\documentclass{ecai} 
\usepackage{latexsym}
\usepackage{amssymb}
\usepackage{amsmath}
\usepackage{amsthm}
\usepackage{booktabs}
\usepackage{enumitem}
\usepackage{graphicx}
\usepackage{color}
\usepackage{multirow} 
\usepackage{algorithm}  
\usepackage{algorithmic}

\newtheorem{theorem}{Theorem}
\newtheorem{lemma}[theorem]{Lemma}

\newcommand{\BibTeX}{B\kern-.05em{\sc i\kern-.025em b}\kern-.08em\TeX}

\begin{document}

%%%%%%%%%%%%%%%%%%%%%%%%%%%%%%%%%%%%%%%%%%%%%%%%%%%%%%%%%%%%%%%%%%%%%%%%

\begin{frontmatter}

%%% Use this command to specify your submission number.
%%% In doubleblind mode, it will be printed on the first page.

\paperid{4522} 

%%% Use this command to specify the title of your paper.

\title{Instruction-Based Molecular Graph Generation with Unified Text-Graph Diffusion Model}

%%% Use this combinations of commands to specify all authors of your 
%%% paper. Use \fnms{} and \snm{} to indicate everyone's first names 
%%% and surname. This will help the publisher with indexing the 
%%% proceedings. Please use a reasonable approximation in case your 
%%% name does not neatly split into "first names" and "surname".
%%% Specifying your ORCID digital identifier is optional. 
%%% Use the \thanks{} command to indicate one or more corresponding 
%%% authors and their email address(es). If so desired, you can specify
%%% author contributions using the \footnote{} command.

\author[A]{\fnms{Yuran}~\snm{Xiang}\thanks{Email: xiangyuranyp@stu.pku.edu.cn}}
\author[A]{\fnms{Haiteng}~\snm{Zhao}}
\author[B]{\fnms{Chang}~\snm{Ma}} 
\author[A]{\fnms{Zhi-Hong}~\snm{Deng}\thanks{Email: zhdeng@pku.edu.cn, Corresponding author}}

\address[A]{State Key Laboratory of General Artificial Intelligence, School of Intelligence Science and Technology, Peking University}
\address[B]{The University of Hong Kong}

%%% Use this environment to include an abstract of your paper.

\begin{abstract}
Recent advancements in computational chemistry have increasingly emphasized generating and editing molecules from textual instructions. However, integrating graph generation with instruction understanding remains challenging, as most existing approaches either rely on molecular sequences in text modality with limited structural information, or struggle with multimodal alignment in graph diffusion methods. To address these limitations, we propose $\textbf{UTGDiff}$ (\textbf{Unified Text-Graph Diffusion Model}), a novel framework that utilizes pre-trained language models for discrete graph diffusion, enabling the generation of molecular graphs from instructions. UTGDiff introduces a unified text-graph transformer as a denoising network, adapted with minimal modifications from language models to process graph data via attention bias. Experimental results show that UTGDiff consistently outperforms both sequence-based and conditional graph-diffusion baselines on instruction-based molecule generation and editing tasks with fewer parameters, covering instructions specifying molecular structures or properties.

\end{abstract}

\end{frontmatter}

%%%%%%%%%%%%%%%%%%%%%%%%%%%%%%%%%%%%%%%%%%%%%%%%%%%%%%%%%%%%%%%%%%%%%%%%

\section{Introduction}

Molecules possess structures and properties that determine the characteristics of substances. Research into novel molecules is crucial for fields such as chemistry, biology, and drug discovery \citep{schneider2005computer}. A central challenge in drug discovery is designing molecules that are stable, exhibit desired functionalities, and interact with specific biological targets \citep{hartenfeller2011novo}. Traditionally, this process has been resource-intensive and time-consuming \citep{dickson2009cost}. However, recent advances in deep learning have enabled cost-effective methods, revolutionizing drug design and accelerating the generation of drug-like molecules \citep{nag2022deep,askr2023deep}.

Modern drug discovery requires exploring vast chemical spaces under specific constraints to identify molecules that fulfill complex therapeutic objectives. Traditional approaches based on a few scalar properties (e.g., solubility) often fall short in capturing the diversity and complexity of real-world design targets, such as specific structural motifs or functional groups. In contrast, natural language offers a flexible and expressive medium for conveying high-level design intents that encompass functional and structural goals. Consequently, there is increasing interest in the tasks of \textbf{instruction-based molecule generation and editing} \citep{edwards2022translation,fang2023mol,liu2024graph}, which aims to generate novel candidates from natural language instructions or edit existing ones to meet desired instructions, as illustrated in Figure~\ref{fig:drug}.

\begin{figure}[htbp]
  \centering
  \includegraphics[width = 0.87\columnwidth]{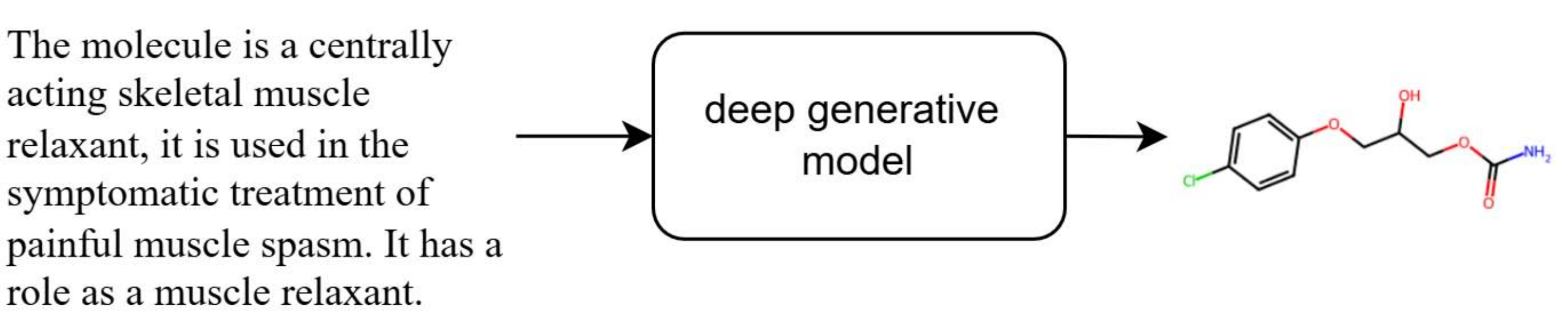}
  \caption{A generative model works like a drug design specialist, with instruction input and molecule output.}
  \label{fig:drug}
\end{figure}

To generate molecules from natural language instructions, prior studies directly use language models on string-based representations like SMILES or SELFIES  \citep{weininger1988smiles,krenn2020self}.  This method has been explored by fine-tuning small-scale language models like T5 \citep{edwards2022translation} or adapting large language models \citep{ye2025drugassist}. However, such single-modal methods face inherent limitations. Notably, molecular strings lack explicit structural information, such as neighborhood relationships based on edges between atoms \citep{jiang2022multigran,wu2023molecular}. Moreover, representing molecules purely in language can lead to deviation from realistic molecular distributions, as minor string reassembly may produce chemically implausible structures \citep{druchok2021toward}. These issues highlight the need for a multimodal framework for language and graph models, combining robust structural representation with strong instruction-following abilities.

To better represent graph structures in multimodal frameworks, several recent studies have explored powerful graph generation techniques. Among them, graph diffusion methods have shown strong performance and scalability on larger molecules \citep{niu2020permutation,jo2022score}. Recent efforts have introduced guiding the diffusion process using several scalar conditions via predictor-guided \citep{vignac2023digress,weiss2023guided} or predictor-free approaches \citep{liu2024graph,zhu20243m}. However, when conditioning on natural language, existing methods face notable limitations. While sentences embeddings can replace scalar conditions, they are typically encoded as fixed vectors and injected into node and edge embeddings through simple operations like addition or linear projection. This prevents graph elements from attending to relevant textual tokens, limiting semantic selectivity and hindering fine-grained injection from text, as nodes and edges cannot determine which instruction components are relevant to themselves. Moreover, language representations remain static throughout denoising, preventing interaction and evolving with graph, thereby limiting the model’s ability to refine textual representations and provide guidance tailored to incomplete structures in denoising. These limitations highlight the need for finer-grained and bidirectional integration between language and graph.

\begin{figure*}[htbp]
  \centering
  \includegraphics[width=0.88\textwidth]{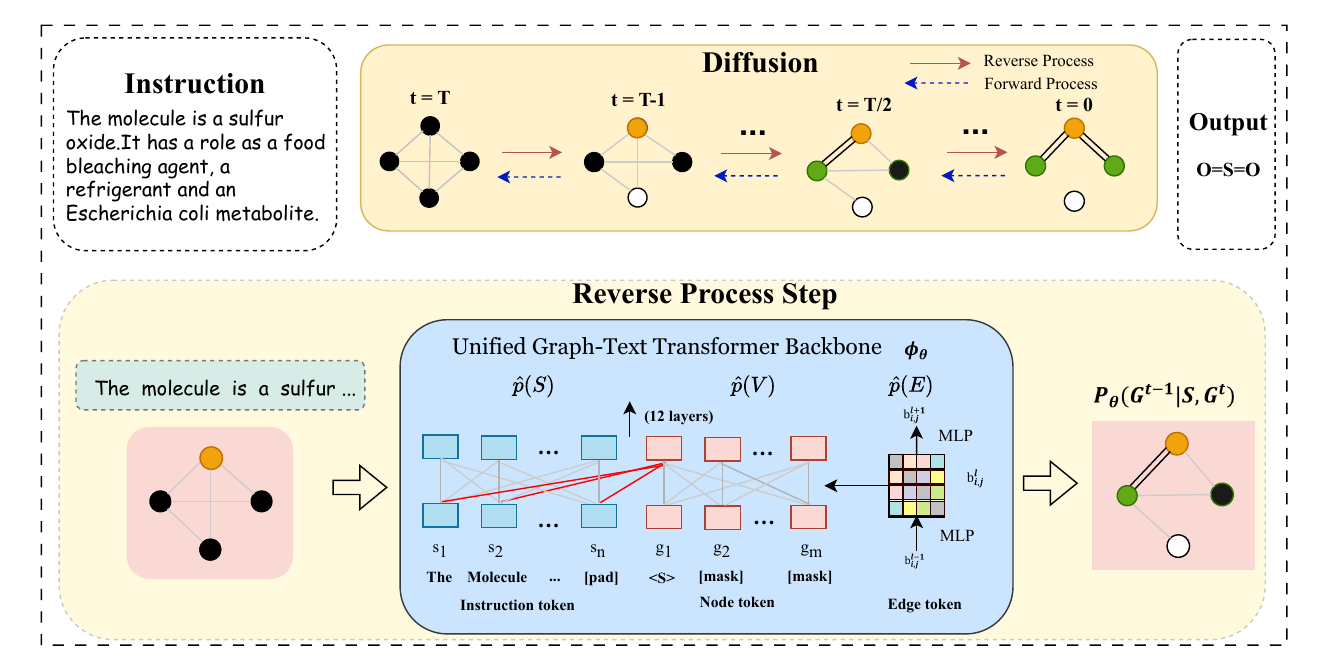}
  \caption{Overview of the UTGDiff framework. It incorporates attention bias into the vanilla transformer, forming a unified text-graph transformer that serves as a denoising network for discrete graph diffusion to generate molecular graphs from instructions. Noise decays some nodes and edges into [MASK] during training (forward process), with the reverse process aiming to recover original graphs as the training objective. Sampling starts with a masked graph  and iterates $T$ times to reduce noise, where each time it process together with fixed given instruction.} 
  \label{fig:example-diff}
\end{figure*}

To this end, we propose the \emph{Unified Text-Graph Diffusion Model (UTGDiff)}, a novel conditional graph diffusion framework that unifies language and graph modalities within a single language transformer. By embedding both modalities into a unified transformer, UTGDiff enables more effective token-level text injection, as graph tokens can selectively attend to text tokens via attention. In addition, the instruction representation is continuously updated through attention with graph tokens, allowing it to evolve with molecular structures and contribute back to graph generation. UTGDiff extends masked language models to simulate graph diffusion through three key novel modifications, treating [MASK] tokens as noise in graph, extending the vocabulary with atom-specific tokens, and introducing a loss function to preserve text comprehension for alignment. We further incorporate attention bias into the transformer to model edges. 

We conduct extensive experiments on the pre-trained UTGDiff model across instruction-based molecule generation and editing tasks, covering conditions on both molecular structures and properties. Our results demonstrate that, despite its small parameter size, UTGDiff consistently outperforms both language model and graph diffusion baselines under comparable pretraining resources available to us, achieving higher similarity while maintaining a relatively high chemical validity. In summary, our main contributions are:\begin{itemize}
\item We propose UTGDiff, a novel framework that unifies language and graph modalities within a single language transformer, with graph selectively attending to text and text being refined by graph.

\item We extend language models to simulate graph diffusion by treating [MASK] as graph noise, extending atom-specific tokens, modifying the loss function, and incorporating attention bias for edges.

\item We conduct extensive experiments showing that UTGDiff outperforms language model and graph diffusion baselines in key areas of drug design, achieving higher similarity and relatively high validity under comparable pretraining resources available to us.
\end{itemize}

\section{Related Works}

\paragraph{Molecule Generation}
\label{sec:2.1}

Early works on molecule generation introduced sequence-based models for SMILES \cite{dai2018syntax,gomez2018automatic}. Due to the lack of structural information, recent studies explored graph-based methods, including auto-regressive \cite{grisoni2020bidirectional, you2018graphrnn}, VAEs \cite{simonovsky2018graphvae,jin2018junction}, GANs \cite{de2018molgan}, and Normalizing Flows \cite{ma2021gf,zang2020moflow}. Graph Diffusion is a relatively new approach with superior performance  \citep{jo2022score,vignac2023digress}, which serves as our foundational framework.

\paragraph{Conditional Molecule Generation}

Conditional molecule generation aims to generate a molecule that satisfies a specific property. Early attempts are based on single scalar property such as dipole moment $\mu$ and p-logP scores \cite{li2018multi,huang2023conditional}. More complex conditions involve multi-condition \cite{bilodeau2022generative,liu2024graph} and textual instructions.  The latter are addressed by sequence-based generation \citep{edwards2022translation,christofidellis2023unifying,li2024empowering,pei2023biot5} and multimodal generation employing a graph decoder \citep{su2022molecular,liu2024git}.

\paragraph{Cross-Modal Molecule Models}
\label{sec:2.2}

In addition to the above instruction-based cross-modal graph generation models, some studies represent molecules and text in cross-modal models for property prediction or text generation \citep{liu2023molca,zhao2024gimlet,liu2023multi}, employing contrastive alignment, cross-modal projector, or unified text-graph backbone.

\paragraph{Discrete Diffusion}
\label{sec:2.3}

Previous Gaussian noise in continuous diffusion models has been shown to be insufficient for discrete data structure. As a solution, researchers introduced the discrete diffusion method, initially discussed in text diffusion \citep{austin2021structured, he2023diffusionbert} and later adapted for graph data \citep{vignac2023digress,kong2023autoregressive}.

\section{Method}

\subsection{Problem Formulation}
\label{3.1}
The objective of our task is to generate a molecular graph $G = (V, E)$ from a given textual description $S$ specifying molecular functions and properties. The graph consists of $m$ nodes $V \in \mathbb{R}^{m \times a}$ and edges $E \in \mathbb{R}^{m \times m \times b}$ representing node and edge attributes, respectively. Each attribute $v_i \in \mathbb{R}^{a}$ and $e_{i,j} \in \mathbb{R}^{b}$ is a one-hot vector corresponding to the categories of the respective nodes and edges. The instruction text is defined as $S = [s_1, \dots, s_n]$, where $s_i$ denotes the $i$-th token and $n$ is its length. The goal is to perform conditional graph generation, i.e., sampling $V, E \sim P (V, E \mid S)$.

\subsection{Overview of Diffusion Framework}
Building upon the above problem setting, We propose a text-conditioned diffusion framework for molecular graph generation, illustrated in Figure \ref{fig:example-diff}. As a generative process, the diffusion model iteratively refines random noise into structured data through a sequence of transformations. It involves a forward process which progressively corrupts the data by adding noise according to a noise model, and a reverse process which progressively denoises the data to reconstruct the original molecular structure using a denoising network. The core component of our framework is the unified text-graph denoising transformer, started from a language-pretrained transformer for text comprehension, and enhanced with a [MASK] noise module and graph-specific modeling techniques. 

\subsection{Forward Process and Noise Model}

We adopt the forward diffusion process following \citet{ho2020denoising}, where the clean graph $G^0$ is progressively corrupted by a noise model $q$, producing a sequence of increasingly noisy graphs $(G^1, \dots, G^T)$. It is defined as $q (G^{1:T} \mid G^0) = \prod^T_{t=1} q (G^t \mid G^{t-1})$.

Following the discrete diffusion setting in \citep{vignac2023digress}, a sequence of transition matrices $ (Q^1, \dots, Q^T)$ is defined independently for nodes $V$ and edges $E$ as noise model, denoted as $Q_V^t$ and $Q_E^t$. Each step in forward process can be expressed as: 

\begin{equation}
q (G^t\mid G^{t-1}) = (V^{t-1}Q_V^t, E^{t-1}Q_E^t)
\end{equation}
Under the cumulative transition matrices formulation $\bar{Q}_V^t=Q_{V}^{1}\dots Q_{V}^{T}$ and $\bar{Q}_E^t=Q_{E}^{1}\dots Q_{E}^{T}$, the forward process can also be expressed as $q (G^{t}\mid G^0)= (V\bar Q_{V}^{t},E\bar Q_{E}^{t})$. With specific transition matrices, the 1-step and t-step marginals are in the closed form, and the posterior at time step $t-1$ admits a closed-form expression:
\begin{equation}
\begin{split}
q(v_i^{t-1}|v_i,v_i^t) = (v_i^{t-1} {Q_V^t}^\top \circ v_i \bar{Q}_V^{t-1})/(v_i \bar{Q}^t_V {v_i^{t}}^\top) \\
q(e_{ij}^{t-1}|e_{ij},e_{ij}^t) = (e_{ij}^{t-1} {Q_V^t}^\top \circ e_{ij} \bar{Q}_V^{t-1})/(e_{ij} \bar{Q}^t_V {e_{ij}^{t}}^\top)
\end{split}
\end{equation} 

\paragraph{Graph Diffusion with [MASK] absorbing states}

Then we should focus on the formulation of noise model. Traditional methods consider the noise model using the marginal distributions over atom and bond types \citep{vignac2023digress}. While they claimed this makes training easier theoretically and experimentally, such assumptions are tailored to single-modal graph denoising model trained from scratch. In contrast, our goal is to integrate language understanding and graph denoising within a unified language-pretrained transformer, rendering this noise modeling assumption incompatible with our architecture.

Therefore, we employ [MASK] absorbing states originally proposed in text diffusion \citep{austin2021structured} into our graph diffusion framework. Under this formulation, each element $v_i^t$ and $e_{i,j}^t$ independently decays into a [MASK] token as an absorbing state according to a noise schedule. Following the schedule $\beta (t) = (T-t+1)^{-1}$, the transition matrix $Q^t$ is formally expressed as follows for both nodes and edges simultaneously, where $z$ denotes the [MASK] index, and $x,y$ denote categories at consecutive timesteps:
\begin{align}
    [Q^{t}]_{x,y}&=\left\{\begin{array}{l l l}{{1}}&{{\text{if }}}{{x=y=z}}\\ {{1-\beta (t)}}&{{\text{if }}}{{x=y\not=z}}\\ {{\beta (t)}}&{{\text{if }}}{{y=z,x\not=z}}\\{{0}}&{{\text{otherwise}}}\end{array}\right.
\end{align} 

The feasibility of this method is supported by the interpretation of BERT as a one-step discrete diffusion model \cite{austin2021structured}, where the network is trained to predict the original tokens from [MASK] inputs. As we adopt a language-pretrained transformer familiar with this objective, the formulation aligns with its pretraining goal, thereby facilitating faster convergence of graph diffusion during fine-tuning. It also promotes better integration of text and graph modalities, and helps retain language understanding through direct parameter reuse. Empirical results further validate its effectiveness.

\subsection{Reverse Process and Denoising Network}
\label{sec:rev}

With the initial noisy graph $G^T$, the reverse process generates $G^0$ iteratively in reversed steps $t = T, T-1, ... , 0$, formally represented as $p_\theta (G^{0:T-1}\mid G^T, S) = \prod^T_{t=1} p_\theta (G^{t-1} \mid G^t, S)$. It relies on a denoising neural network $\phi_\theta$ parameterized by $\theta$, which is trained to invert the forward process by predicting $G^{0}$ from $G^t$ conditioned on the instruction $S$ for each $t$, as the product over nodes and edges: $p_\theta (G^{0}\mid G^t, S) = \prod_{v\in V} p_\theta (v \mid G^t, S) \prod_{e\in E} p_\theta (e \mid G^t, S)$. 

Then, to calculate the reverse distribution $p_\theta (G^{t-1}\mid G^t, S)$, we adopt $x_0$-parameterization with noise back to the clean graph prediction, represented as $p_\theta (G^{t-1} \mid G^{t}, S) = q (G^{t-1}\mid G^t,\phi_\theta (G^t, S))$, and marginalized over predictions of node and edge types $\mathcal{V'}, \mathcal{E'}$:

\begin{equation}
\begin{split}
   &\!p_{\theta} (v^{t-1}_i\!\mid\!G^{t}\!,\! S)\!=\!\sum_{v\in\mathcal{V'}}q (v^{t-1}_i \!\mid\! v_{i}\!=\!v,\!G^{t})\;p_{\theta}(v_i=v|G^t,S) \\
   &\!p_{\theta} (e^{t-1}_{i,j}\!\mid\!G^{t}\!, \!S)\!=\!\sum_{e\in\mathcal{E'}}q (e^{t-1}_{i,j}\!\mid\!e_{i,j}\!=\!e,\!G^t)\;p_{\theta}(e_{i,j}=e|G^t,S)
\end{split} 
\end{equation}

To enable graph tokens to selectively attend to text tokens, and allow textual representations to be refined by graph structure, we propose the \emph{Unified Text-Graph Transformer} as our denoising network, which unifies two modalities within a single language-pretrained transformer. Below, we describe two key components of the model.

\paragraph{Molecule Tokenization via Shared Vocabulary}

To adapt the language transformer for graph data, we first address the representation of molecular nodes. The key novel design is to treat molecular graph nodes as discrete tokens drawn from a shared vocabulary space, enabling seamless integration with transformer-based language models. 

Specifically, we assign a unique textual token (e.g., [C], [O-]) to each atomic category and embed them in the same token space. The initial embeddings are defined as $h^0 = \text{Emb}(S+V) = [h_1, \dots, h_n, h_{n+1}, \dots, h_{n+m}] \in \mathbb{R}^{(n+m) \times d_h}$, corresponding to $n$ instruction tokens and $m$ graph nodes. The graph node embeddings $h_{n+1}, \dots, h_{n+m}$ are derived from the atom tokens (e.g., [[S], [O], [O]] for sulfur oxide).

This tokenization strategy allows graph nodes to be processed identically to language tokens within a unified transformer, enabling attention-based cross-modal integration and consistent parameter sharing. It also prevents compositional ambiguities, such as mistakenly interpreting cobalt (Co) as a combination of carbon and oxygen.

\setlength{\tabcolsep}{2pt} 

\begin{table*}

\centering
\begin{tabular}{c c c|c c c c c c}
\toprule
Type & Model & \#params & MACCS FTS $\uparrow$ & RDK FTS $\uparrow$ & Morgan FTS $\uparrow$ & FCD $\downarrow$ & Exact $\uparrow$ & Valid $\uparrow$ \\
\midrule
\multirow{3}*{\shortstack{Text \\ Auto-regressive\\ (w.o. pretrain)}} 
& T5-base & 248M & 0.731 & 0.605 & 0.545 & 2.48 & 0.069 & 0.660 \\
& T5-large & 783M & 0.823 & 0.731 & 0.670 & 1.22 & 0.279 & 0.902 \\
& BioT5-base (reproduce) & 252M & 0.821 & 0.708 & 0.633 & 1.67 & 0.071 & \textbf{1.000} \\
\midrule
\multirow{4}*{\shortstack{Text \\ Auto-regressive \\ (pretrain)}} 
& MolT5-base & 248M & 0.721 & 0.588 & 0.529 & 2.18 & 0.081 & 0.772 \\
& MolT5-large & 783M & 0.834 & 0.746 & 0.684 & 1.20 & 0.311 & 0.905 \\
& MolXPT & 350M & 0.859 & 0.757 & 0.667 & \textbf{0.45} & 0.215 & 0.983 \\
& BioT5-base (reproduce) & 252M & 0.843 & 0.745 & 0.676 & 1.41 &  0.097 & \textbf{1.000} \\
\midrule
\multirow{2}*{Text diffusion} & tgm-dlm & 125M & 0.854 & 0.739 & 0.688 & 0.77 & 0.242 & 0.871 \\
& tgm-dlm w/o corr & 125M & 0.874 & 0.771 & 0.722 & 0.89 & 0.242 & 0.789 \\
\midrule
\multirow{4}*{LLM} & Chatgpt3.5 (0-shot) & - & 0.703 & 0.568 & 0.517 & 2.49 & 0.079 & 0.721 \\
& Chatgpt3.5 (10-shot) & - & 0.847 & 0.708 & 0.624 & 0.57 & 0.139 & 0.887 \\
&gpt4 (0 shot) & - &	0.787	&0.577&	0.492&	1.33	&0.055	&0.883\\
&gpt4 (10 shot)	& - &0.872&	0.736&	0.659&	0.66&	0.092&	0.891\\
&LlaSMol-Mistral & 7B & 0.853&0.726&0.650&0.70&0.132&	0.935\\
\midrule
\multirow{4}*{Graph-based} & Momu-S & 113M & 0.244	& 0.103 & 0.047 & 22.21 & 0.000 & 1.000 \\
& Digress (similarity guidance) & 289M & 0.577 & 0.389 & 0.288 & 28.82 & 0.014 & 0.854 \\
& Graph-DiT & 162M & 0.374 & 0.269 & 0.159 & 18.58 & 0.000 & 0.909 \\
& 3M-Diffusion & 162M & 0.548 & 0.370 & 0.273 & 4.42 & 0.005 & 1.000 \\
\midrule
 & UTGDiff (w.o. pretrain) &  125M & 0.867 & 0.763 & 0.695 & 0.92 & 0.227 & 0.856 \\
& UTGDiff (pretrain) &  125M & \textbf{0.885} &\textbf{ 0.795} & \textbf{0.724} & 0.86 & \textbf{0.374} & 0.893 \\
\bottomrule

\end{tabular}
\caption{Results of the instruction-based molecule generation on ChEBI-20 for both with/without pretraining setting.}
\label{table2}
\end{table*}

\paragraph{Attention Bias}

Next, we address the representation of edges. In the attention layer with parameters for values, keys, queries, and outputs, i.e., $W_V, W_Q, W_K \in \mathbb{R}^{d_h \times d_k}, W_O \in \mathbb{R}^{d_k \times d_h}$, we introduce an additional bias term to incorporate edge information. The attention score between the $i$-th and $j$-th graph tokens is now formalized as:
\begin{equation}
\begin{split}
   \hat{A}_{i,j}^l = \frac{1}{\sqrt{d_{k}}}\left (h_{i}^lW_{Q}\right)\left (h_{j}^lW_{K}\right)^\top+b_{i,j}^l, \\ 
A^l = \mathrm{softmax} (\hat{A}^l), \ \text{Attn} (H)= A^l H^l W_V W_O
\end{split}
\end{equation} 
Here, $b_{i,j}$ denotes the bias added to the attention scores between graph tokens to incorporate structural information \citep{ying2021transformers,zhao2023more}. To ensure the expressive power in distinguishing neighbors as one-layer GNNs and the symmetry between edge representations, we therefore represent the initial embedding from the edge category as $b_{i,j}^0$ and the update of attention bias as follows: 
\begin{align}
b_{i,j}^l =
\begin{cases}
b_{i,j}^0, & \text{if } l = 0 \\
\text{FFN}((A_{i,j}^{l-1} + {A_{j,i}^{l-1}})/2) , & \text{otherwise} \\
\end{cases}
\end{align} 
The FFN module applied to attention scores follows the structure in standard transformer, including a residual connection and LayerNorm. Our design distinguishes neighbors using edge-specific initial biases $b_{i,j}^0$, and enforces symmetry by symmetrizing the attention scores. Beyond structural constraints, our model minimally modifies standard transformer, enabling joint text-graph processing without additional GNNs and allowing full reuse of pre-trained language model weights. Also, layer-wise text-graph interactions enable graph tokens to selectively attend to text tokens and text representations to be refined by graph, supporting fine-grained interactions.

\paragraph{Output} After obtaining the final hidden states $h^L$ for tokens and $b^L$ for edges through stacked attention layers, we map them back to the vocabulary space for the probability of each category through two separate masked language model heads at the end of the transformer. This yields the node logits $\log p_\theta(v_i \mid G^t, S)$ and edge logits $\log p_\theta(e_{i,j} \mid G^t, S)$ for the clean graph distribution. Assembling them into matrices $\hat{p}(V) \in \mathbb{R}^{m \times a}$ and $\hat{p}(E) \in \mathbb{R}^{m \times m \times b}$, the unified denoising network is summarized as:
\begin{equation}
\phi_\theta (G^{t},S) = \hat{p} (G) = [\hat{p} (V), \hat{p} (E)]
\end{equation}

\paragraph{Sampling} 

Empirically, positional embeddings can improve performance in the unified transformer by atom locations, while inherently break permutation equivariance. Therefore, inspired by permutation-based graph encoding methods \cite{murphy2019relational,nikolentzos2020random}, we adopt a sampling strategy that preserves equivariance at inference: positional embeddings are applied to a fixed node order during training, while random permutations of position indices are introduced during sampling.

Therefore, for any permutation $\pi$, the model generates a graph with node $V$, adjacency matrix $A$, and position indices $i$ with equal probability, satisfying $\mathbb{P}_i (V, A) = \mathbb{P}_{\pi^T i} (\pi^T V, \pi^T A \pi)$, thereby ensuring equal probability over the molecule's permutation set during sampling.  We prove the equivariance in Appendix B in the supplementary material \cite{xiang2024instruction}.

\subsection{Training Objective}
To train the denoising network, the original method \cite{vignac2023digress} optimizes the cross-entropy loss (denoted as $\mathrm{CE}$) between the predicted probabilities $\hat{p}(G)$ and the clean graph $G$ in the context of single-modal graph diffusion. However, in our multimodal setting, such graph-only supervision is insufficient. Since the model is required to interpret textual instructions while generating molecular structures, an additional loss term related to the text modality is necessary to preserve its language understanding capabilities.

Therefore, we introduce a novel additional textual loss by applying the cross-entropy objective between the predicted logits $\hat{p}(S)$ and the ground-truth instruction tokens. Through the denoising transformer, the logits $\hat{p}(S)$ can obtained alongside $\hat{p}(V)$ and $\hat{p}(E)$ via the masked language model head, formulated as $\phi_\theta (G^{t}, S) = [\hat{p} (S), \hat{p} (V), \hat{p} (E)]$. Then we define the final training loss $\mathcal{L}$ as:
\begin{equation}
\begin{split}
\mathcal{L} &= l (\hat{p} (V),V) + l (\hat{p} (E),E) + l (\hat{p} (S),S) \\ \!&=\!\sum_{1 \le i \le m}\!\mathrm{CE}(v_{i},\!\hat{p}_i (V))  \!+\!\sum_{1 \le i,j \le m}\!\mathrm{CE} (e_{i,j},\!\hat{p}_{i,j} (E)) \\ &+ \!\sum_{1 \le i \le n}\!\mathrm{CE}(S_{i},\!\hat{p}_i (S))
\end{split}
\end{equation}

\begin{table*}[t]
\centering
\begin{tabular}{lcccccccccc}
\toprule
{\multirow{2}{*}{Model}}  & \multicolumn{1}{c}{ Validity $\uparrow$} & \multicolumn{4}{c}{ Distribution Learning} & \multicolumn{5}{c}{Condition Control} \\
&  (\small w/o rule checking) & Coverage $\uparrow$ &Diversity $\uparrow$ & Similarity $\uparrow$ & Distance $\downarrow$ &   Synth. $\downarrow$  &  O$_2$Perm $\downarrow$ &N$_2$Perm $\downarrow$ &  CO$_2$Perm $\downarrow$ & Avg. MAE $\downarrow$ \\
\midrule
JTVAE-BO & 1.0000 (N.A.) & 10/11 & 0.7366 & 0.7294 & 23.5990 & {\textbf{1.0714}} & 1.0781 & 1.2352 & 1.0978 & 1.1206 \\
\midrule
DiGress & 0.9913 (0.2362) & 11/11 & 0.9099 & 0.2724 & 22.7237 & 2.9842 & 1.7163 & 2.0630 & 1.6738 & 2.1093 \\
GDSS & 0.9205 (0.9076) & 9/11 & 0.7510 & 0.0000 & 34.2627 & 1.3701 & 1.0271 & 1.0820 & 1.0683 & 1.1369 \\
Graph-DiT-LCC & 0.9753 (0.8437) & 11/11 & 0.8875 & 0.9560 & 7.0949 & 1.3099 & 0.8001 & 0.9562 & 0.8125 & 0.9697 \\
Graph-DiT & 0.8245  (0.8437) & 11/11 & 0.8712 & { {0.9600}} & { {6.6443}} & 1.2973 & { {0.7440}} & {{0.8857}} & {{0.7550}} & {{0.9205}} \\
\midrule
UTGDiff (Ours) & 0.9596  (0.9085) & 11/11 & 0.8532 & { \textbf{0.9637}} & { \textbf{6.0923}} & \textbf{1.0208} & { \textbf{0.6462}} & {\textbf{0.7826}} & {\textbf{0.6459}} & {\textbf{0.7739}} \\
\bottomrule
\end{tabular}
\caption{Multi-Conditional Generation of 10K Polymers: Results on the synthetic score (Synth.) and three numerical properties (gas permeability for O$_2$, N$_2$, CO$_2$). MAE is calculated between the input conditions and the properties of the generated polymers using Oracles. Best results are {\textbf{highlighted}}.}
\label{tab:main1}
\end{table*}

This loss helps preserves the model’s language comprehension by reinforcing masked language modeling behavior. When a text token is provided as input, the model is explicitly trained to predict the same token, ensuring token-level consistency. This aligns with the core objective of MLM: reconstructing masked tokens from context demonstrates semantic comprehension, while preserving the identities of unmasked tokens stabilizes the language representations and supports the retention of language understanding during training. 

\subsection{Pretraining Method}

We have mentioned that BERT can be interpreted as a one-step [MASK] absorbing diffusion model which discussed in~\citep{austin2021structured}. This shared training objective enables initialization from pre-trained language models~\citep{he2023diffusionbert}. Accordingly, we initialize our model from the pre-trained RoBERTa~\citep{liu2019roberta} to preserve language comprehension.

However, RoBERTa initialization is insufficient for masked prediction on novel graph tokens. To address this, we further pretrain the model on masked graph data while preserving its text modeling ability. We collect paired and single-modal data from textual and graph modalities, randomly mask 15\% of tokens and edge indices, and train the model with masked language modeling. The effectiveness of this method on graph data has been validated by prior work~\citep{hou2022graphmae}.

\begin{table}[t]

\centering
\begin{tabular}{c |c c}
\toprule
Model & names \\
\midrule

molT5 & C4, ZINC \\

\multirow{2}*{BioT5} & C4, Pubmed, ZINC (w./w.o. IUPAC), \\ & pubchem324K, bioRxiv, NER Wrapped biotext \\

UTGDiff & Pubmed, ZINC, pubchem324K \\
\bottomrule

\end{tabular}
\caption{Review of pretraining datasets.}
\label{table0}
\end{table}

Specifically, for multimodal data, we use the PubChem-324K~\citep{liu2023molca}, containing 320K molecule-text pairs. For single-modal data, we employ molecular graphs from ZINC20~\citep{irwin2020zinc20} and biomedical text from PubMed abstracts~\citep{white2020pubmed}, totaling nearly 100 million entries. Since PubChem-324K is explicitly curated to avoid overlap in generation and others are single-modal data, there're no risk of data leakage.

\section{Experiment}

\subsection{Instruction-based Molecule generation}
\label{sec:4.1}

First, we evaluate UTGDiff on molecule generation tasks, to demonstrate whether UTGDiff achieves high instruction fidelity and structural validity under structural descriptions. 

\paragraph{Dataset} We utilize the ChEBI-20 dataset~\citep{edwards2021text2mol} for fine-tuning and evaluation. It contains 33,010 molecule-instruction pairs, with 10\% allocated for validation and 10\% for testing.

\paragraph{Baselines} Baseline models include (1) Sequence-based models---T5, MolT5 \citep{edwards2022translation}  and BioT5 \citep{pei2023biot5}; (2) Large language models---ChatGPT-3.5, GPT-4 (2024-04-09), and LlaSMol-Mistral-7B fine-tuned on molecular data; (3) text diffusion model tgm-dlm \citep{gong2024text}; and (4) graph-based models---MoMu-S \cite{su2022molecular} (flow-based graph decoder), Digress\cite{vignac2023digress} (predictor-guided  diffusion), Graph-DiT \cite{liu2024graph} (predictor-free multi-condition diffusion), and 3M-Diffusion \cite{zhu20243m} (separate encoders/decoders).  BioT5 is retrained on our pretraining data for fair comparison (Table~\ref{table0}) due to unavailable original pretraining data.

\paragraph{Evaluation} We adopt 6 metrics to assess the similarity and quality of generated molecules: (1) Exact match  (Exact); (2) Molecular validity (Validity); (3) 3 fingerprint metrics for similarity in functional group level (FTS); (4) Fréchet ChemNet Distance for similarity in set level (FCD). The full description is in Appendix C in the supplementary material \cite{xiang2024instruction}.

\paragraph{Results} The results are summarized in Table~\ref{table2}.  Additional results with different random seeds and hyperparameter details are provided in Appendix C and D in the supplementary material \cite{xiang2024instruction}. Our main findings are as follows:

\textbf{UTGDiff outperforms all text-based autoregressive methods on similarity metrics with comparable pretraining data size.} 
It achieves substantial improvements in Exact match and FTS scores, surpassing text baselines by 6\% on average, highlighting the diffusion paradigm's strength in capturing complex molecular structures.

\textbf{UTGDiff significantly outperforms graph-based methods.} 
Previous graph diffusion approaches struggle with complex conditional generation due to information bottleneck and limited scalability. In contrast, UTGDiff enables finer-grained integration of language and graph modalities through unified model and layer-wise aggregration, leading to substantial improvements in instruction-following ability.

\textbf{Simple text diffusion methods fall short without deep text-graph integration.} 
UTGDiff improves the Exact match score over 15\% compared to text diffusion models, indicating that our gains are not solely attributable to the diffusion framework itself but to the deep structural integration of graphs within the language model.

\textbf{UTGDiff achieves better performance with significantly fewer parameters.}  Our model contains only 125M parameters, substantially smaller than other models but with better performance.

UTGDiff is marginally outperformed by BioT5+ and a few fine-tuned LLM~\citep{li2025large}, which can be attributed to their much larger pretraining corpora. Minor shortcomings in molecular validity stem from inevitable valence violations, while lower FCD can due to the metric's inability to distinguish cross-instruction mismatches, with further analysis in Appendix A in the supplementary material \cite{xiang2024instruction}.

\begin{figure}[htbp]
  \centering
  \includegraphics[width=\columnwidth]{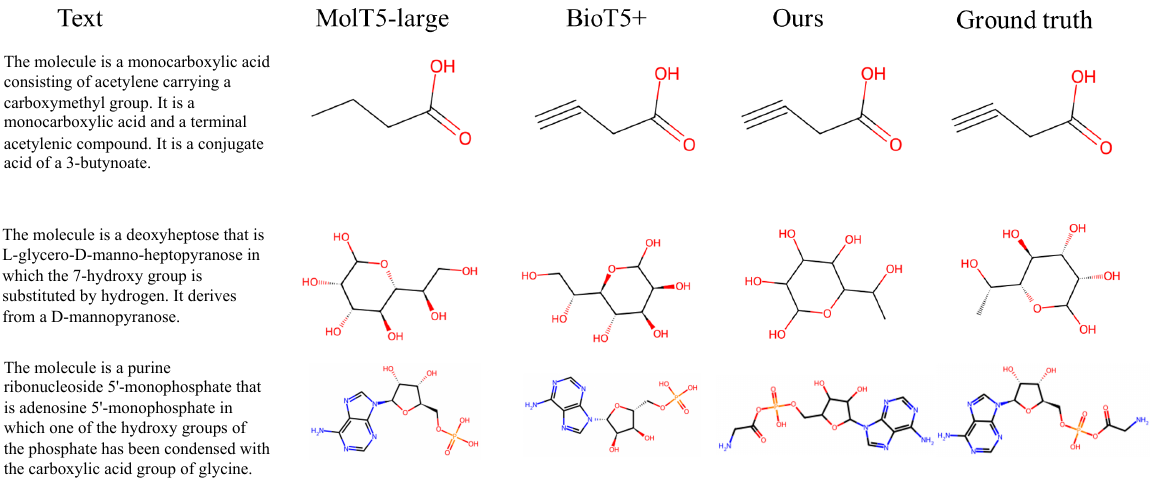}
  \caption{Some generation examples comparison on CHEBI-20.}
  \label{fig:example-example}
\end{figure}

\paragraph{Case Study}  Several examples are shown in Figure~\ref{fig:example-example}. UTGDiff successfully generates accurate molecules as baseline models (Example 1), and can produce correct structures where other models fail (Example 2). Furthermore, it can capture critical instruction details more effective, such as "carboxylic acid group of glycine", demonstrating its ability for contextually appropriate outputs (Example 3).

\begin{table}

\centering
\begin{tabular}{c |c c c c c}
\toprule
Length & 1-64 & 65-96 & 97-128 & 129-160	& $>$ 160  \\
\midrule
BioT5	& 0.810	& 0.866	& 0.859	& 0.817	& 0.788 \\
UTGDiff	& 0.860 & 0.903 & 0.901 & 0.853 & 0.821 \\
\bottomrule
\end{tabular}

\begin{tabular}{c |c c c c}
\toprule
BertzCT & <100 & 100-300 & 300-800 & >800	\\
\midrule
BioT5 & 0.765	& 0.863	& 0.831	& 0.857	\\
tgl-dlm & 0.757	& 0.859	& 0.841	& 0.882	\\
UTGDiff	& \textbf{0.837} & \textbf{0.898} & \textbf{0.881} & \textbf{0.889}\\

\bottomrule

\end{tabular}
\caption{Analysis for scalability of the model}
\label{scalability} 
\end{table}

\paragraph{Scalability} We further assess UTGDiff's scalability on instruction and molecule complexities. We report MACCS FTS across instruction length and BertzCT intervals, a topological  metric reflecting graph size and branching. Table \ref{scalability} shows that UTGDiff maintains strong performance across all complexity levels, demonstrating robustness on detailed instructions and complex molecules.
\begin{table}

\centering
\begin{tabular}{c |c c c c c c}
\toprule
Model & MACCS $\uparrow$ & RDK $\uparrow$ & Morgan $\uparrow$ & FCD $\downarrow$ & Exact $\uparrow$ & Valid $\uparrow$ \\
\midrule
MolT5-base & 0.733	& 0.651	& 0.621	& 2.20	& 0.204 & 0.761 \\
BioT5-base	& 0.737	& 0.646	& 0.595	& \textbf{1.40}	& 0.242 & \textbf{1.000} \\
tgm-dlm & 0.741	& 0.667	& 0.612	& 2.06	& 0.221 & 0.765\\
3M-Diffusion & 0.495 & 0.332 & 0.242 & 3.47 & 0.015 & \textbf{1.000}\\
UTGDiff	& \textbf{0.763} & \textbf{0.675} & \textbf{0.623} & 2.03 & \textbf{0.386} & 0.847 \\
\bottomrule
\end{tabular}

\caption{Analysis for generalization in zero-shot experiment}
\label{scalability} 
\end{table}

\paragraph{Generalization} To assess the model's generalization to unseen instructions, we conduct a zero-shot evaluation on PCDes using the model trained only on ChEBI. Sentence-BERT confirms that PCDes instructions differ from ChEBI's, with 0.929 cosine similarity compared to 0.999. UTGDiff performs competitively across similarity metrics and generalizes well without retraining.

\subsection{Multi-Property Conditional Molecule Generation}
\label{sec:4.7}

The instruction dataset above mainly focuses on molecular structures. However, real-world molecular generation scenarios often require inverse design, based on specified property conditions. To better reflect this, we evaluate multi-property conditional generation, formulating numerical properties as part of instructions to assess our model’s ability across more practical conditions. We adapt UTGDiff by expressing property constraints textually and introducing separate MLP encoders to map each condition into an input embedding.

\paragraph{Dataset} We utilize the polymer dataset for materials, which includes three numerical gas permeability properties: O\textsubscript{2}Perm, CO\textsubscript{2}Perm, and N\textsubscript{2}Perm.  It contains 553 entries, with 20\% for validation and 20\% for testing following the data split protocol in~\cite{liu2024graph}.

\paragraph{Baselines} We compare against several baselines, including the VAE-based JTVAE model with Bayesian optimization (JTVAE-BO), and diffusion-based methods such as predictor-guided GDSS, DiGress~\cite{vignac2023digress}, and predictor-free Graph-DiT~\cite{liu2024graph}. We follow the baseline setup in Graph-DiT~\cite{liu2024graph}, restricting comparisons to graph-based models, as text-based models are not designed for precise numerical control on property-conditioned generation.

\paragraph{Evaluation} For fair comparison, we adopt the same evaluation protocol on 10,000 generated molecules. In addition to validity and FCD metrics described above, we have: (1) heavy atom type coverage (Coverage); (2) diversity among examples (Diversity); (3) fragment-based similarity with the reference set (Similarity); (4) MAE between the generated and conditioned molecules, including synthetic accessibility score (Synth.) and MAE for numerical conditions (Property).

\paragraph{Results} Our findings are shown in Table \ref{tab:main1} and described as follow:

\textbf{UTGDiff maintains strong molecular distribution modeling beyond prior diffusion methods.} UTGDiff achieves the best performance in both similarity and distribution-level metrics. This is notable since it contains only 11 heavy atom types and 5 condition types, much simpler than the diversity seen during pretraining. These results highlight UTGDiff's generalization from large-scale structure-focused pretraining to property-conditioned generation.

\textbf{UTGDiff advances conditional graph diffusion and achieves SOTA multi-property control.} Across all three properties, UTGDiff consistently reduces MAE, averaging a 19\% improvement over the previous SOTA Graph-DiT. This demonstrates that our layer-wise interaction and dedicated numerical encoding enable better multi-property control, validating UTGDiff’s scalability for more diverse property conditions and reflecting real-world generation scenarios.

\begin{figure}[htbp]
  \centering
  \includegraphics[width=0.85\columnwidth]{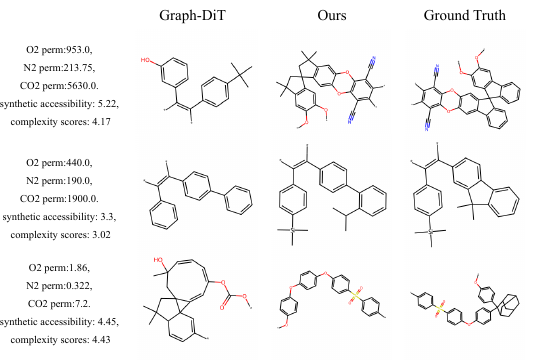}
  \caption{Some generation examples comparison on O2-N2-CO2 Perm.}
  \label{fig:example-new}
\end{figure}

\paragraph{Case Study} Several examples are shown in Figure~\ref{fig:example-new}. UTGDiff successfully generates near-identical molecules that match the ground truth and outperform baseline models (Example 1). Moreover, it accurately recovers correct atom types (S,Si) under property constraints where other models fail (Example 2 and 3), highlighting UTGDiff’s precise understanding of numerical conditions. 

\subsection{Instruction-Based Molecule Editing}
\label{sec:4.4}

We further evaluate UTGDiff on molecule editing tasks to assess  generalization in diverse scenarios. Unlike generation, editing requires an additional source graph as input, incorporated into the instruction in same graph format as output. Evaluation follows the same metrics as in instruction-based molecule generation.

\paragraph{Dataset} We select the retrosynthesis and forward reaction prediction tasks from Mol-Instructions \citep{fang2023mol}. The retrosynthesis task predicts precursors from a target compound, while the forward reaction prediction task forecasts products from given reactants and reagents. Each dataset contains 120K training pairs and 1,000 testing pairs.

\paragraph{Baseline} For the editing datasets, the baseline models include (1) specialist auto-regressive models such as BioT5 \citep{pei2023biot5} and TEXT+CHEM T5 \citep{christofidellis2023unifying}; (2) LLM such as InstructMol \citep{cao2025instructmol} and others. We also reproduce BioT5's experiments on our pretraining datasets rather than referencing their reports as discussed above.

\paragraph{Results} The results are presented in Table \ref{editing1} and Table \ref{editing2}.

\begin{table}

\centering

\begin{tabular}{c c|c c c c c}
\toprule
Type & Model & MACCS$\uparrow$ & RDK$\uparrow$ & Morgan$\uparrow$ & Exact$\uparrow$ & Valid$\uparrow$ \\
\midrule
\multirow{5}*{LLM} & GALACTICA & 0.274 & 0.167 & 0.134 & 0.000 & 0.986 \\

& Mol-Instructions & 0.487 & 0.283 & 0.230 & 0.009 & 1.000 \\

& Llama-7b* (LoRA) & 0.294 & 0.136 & 0.106 & 0.000 & 1.000 \\

& InstructMol-G & 0.523 & 0.422 & 0.285 & 0.114 & 1.000 \\

& InstructMol-GS & 0.852 & 0.753 & 0.714 & 0.407 & 1.000 \\
\midrule
\multirow{2}*{T5} & TEXT+CHEM T5 & 0.765 & 0.685 & 0.585 & 0.141 & 0.698 \\

& BioT5 (reproduce) &  \textbf{0.904} & 0.843 &  0.810 & 0.480 & \textbf{1.000} \\
\midrule
 & UTGDiff  & \textbf{0.904} & \textbf{0.847} & \textbf{0.817} & \textbf{0.541} & 0.945  \\
\bottomrule

\end{tabular}

\caption{Mol-Instruction Retrosynthesis results.}
\label{editing1}
\end{table}
\begin{table}

\centering
\begin{tabular}{c c|c c c c c}
\toprule

Type & Model & MACCS$\uparrow$ & RDK$\uparrow$ & Morgan$\uparrow$ & Exact$\uparrow$ & Valid$\uparrow$ \\
\midrule
\multirow{5}*{LLM} & GALACTICA & 0.127 & 0.036 & 0.051 & 0.000 & 0.995 \\

& Mol-Instructions & 0.509 & 0.313 & 0.262 & 0.045 & 1.000 \\

& Llama-7b* (LoRA) & 0.649 & 0.499 & 0.407 & 0.012 & 1.000 \\

& InstructMol-G & 0.717 & 0.519 & 0.457 & 0.114 & 1.000 \\

& InstructMol-GS & 0.878 & 0.776 & 0.741 & 0.407 & 1.000 \\
\midrule
\multirow{2}*{T5} & TEXT+CHEM T5 & 0.789 & 0.705 & 0.652 & 0.141 & 0.698 \\

& BioT5 (reproduce) & 0.954 & 0.907 & 0.890 & 0.684 & 1.000 \\
\midrule
 & UTGDiff  & \textbf{0.973} & \textbf{0.943} & \textbf{0.942} & \textbf{0.825} & 0.972 \\
\bottomrule

\end{tabular}

\caption{Mol-Instruction forward reaction prediction results.}
\label{editing2}
\end{table}

UTGDiff consistently outperforms all baselines across fingerprint similarity and exact match metrics, achieving 5\% improvement in average in two tasks. These results underscore the methodological advantages of our diffusion approach over autoregressive baselines.

Compared to InstructMol-G and InstructMol-GS, we achieves a twofold increase in exact match scores on forward reaction prediction, along with at least a 5\% improvement across all metrics on both tasks. This highlights the effectiveness of graph diffusion generation, as InstructMol-G integrates a 2D  graph encoder but still relies on a sequence-based decoder, limiting its structural modeling capacity.

\subsection{Ablation Study}
\label{sec:4.5}

We ablate our model on the model design and training strategy on CHEBI-20 dataset,  validating the contribution of these key designs.

\paragraph{Model Design}

We conduct ablation studies focusing on two aspects of model design: (1) the contribution of the diffusion module to generation performance; and (2) the effectiveness of key architectural components, including unification via attention bias and the absorbing state implemented through [MASK] tokens.

\begin{table}
\centering
\begin{tabular}{c c|c c c}
\toprule

model & steps & MACCS FTS$\uparrow$ & Valid $\uparrow$ & time (secs/sample) $\downarrow$  \\
\midrule
\multirow{4}*{UTGDiff}
& 1 & 0.838 & 0.515 & 0.861 \\
& 10 & 0.879 & 0.860 & 0.968  \\
& 100 & 0.885 & 0.893 & 1.929\\
& 1000 & 0.881 & 0.901  & 11.388\\
\bottomrule

\end{tabular}

\caption{Ablation on diffusion. Steps = 1 means without diffusion.}
\label{sample} 

\end{table}

(1) \textbf{Effect of the Diffusion Module.} The $x_0$-parameterization  allow us to use flexible step-size inference during generation. Performing generation with one-step logits represents a setting without diffusion. As shown in Table~\ref{sample}, the diffusion module improves performance, demonstrating the importance of progressive refinement.

\begin{table}

\centering
\begin{tabular}{c|c c c}
\toprule

Model & MACCS FTS$\uparrow$ & FCD $\downarrow$ & Valid $\uparrow$ \\
\midrule
Cross-attn & 0.693 & 2.720 & 0.632 \\

marginal absorbing & 0.877 & 1.852 & 0.837  \\

UTGDiff & \textbf{0.885} & \textbf{0.866} & \textbf{0.893} \\
\bottomrule
\end{tabular}
\caption{Ablation on our model design.} 
\label{des} 
\end{table}

(2) \textbf{Effect of Architecture.} To ablate the unification strategy, we replace attention bias with cross-attention, encoding instructions separately for cross-modal fusion. To ablate the absorbing state design, we replace the [MASK]-based noise representation with marginal distributions as Graph-DiT \citep{liu2024graph}. Table~\ref{des} confirms the effectiveness of both attention bias and [MASK]-based noise modeling.

\begin{table}

\centering
\begin{tabular}{c|c c c c c}
\toprule
Model & MACCS $\uparrow$ & FCD $\downarrow$ & Valid $\uparrow$ & Train loss $\downarrow$  &  Entry \\
\midrule

All &  \textbf{0.885} &  \textbf{0.866} &  \textbf{0.893} & \textbf{7.69e-4}  & 100M\\

\midrule

Pair-only & 0.870 & 0.892 & 0.879 & 1.99e-3 & 300K\\

from scratch  & 0.867 & 0.923 & 0.856 & 1.58e-2 & 0\\

\midrule

no textual loss  & 0.882 & 1.224 & 0.833 & 3.57e-3 & 100M \\

\midrule

no position index  & 0.764 & 2.387 & 0.821 & 2.92e-1 & 100M \\

\bottomrule

\end{tabular}

\caption{Ablation study on pretrain data.} 
\label{pretrain_table} 

\end{table}

\paragraph{Training Strategy}
We conducted ablation experiments to assess the impact of pretraining scale and the textual loss $l(\hat{p}(S), S)$ during fine-tuning. Table~\ref{pretrain_table} highlights the importance of both single-modal and paired pretraining datasets, demonstrating that scaling pretraining data significantly benefits downstream performance. The comparison of textual loss shows that incorporating textual loss leads to better convergence and higher validity. The comparison of position index shows the necessity of our sampling stratage. 

\section{Conclusions and Limitations}

We propose UTGDiff, a unified instruction-based framework that integrates graph and text within a single transformer. With absorbing noise, atom-specific tokens, and attention bias, it enables fine-grained graph diffusion generation conditioned on complex instructions. Our model excels in molecule generation and editing, covering instructions for molecular structures or properties, while maintaining a smaller parameter size compared to existing baselines.

Despite these advancements, UTGDiff is limited by the scale of its pretraining corpus, which may constrain generalization to more diverse molecular structures. Future work will focus on scaling pretraining and exploring more advanced discrete diffusion techniques to further enhance instruction-conditioned generation.

%%%%%%%%%%%%%%%%%%%%%%%%%%%%%%%%%%%%%%%%%%%%%%%%%%%%%%%%%%%%%%%%%%%%%%%%

%%% Use this environment to include acknowledgements (optional).
%%% This will be omitted in doubleblind mode.

%%%%%%%%%%%%%%%%%%%%%%%%%%%%%%%%%%%%%%%%%%%%%%%%%%%%%%%%%%%%%%%%%%%%%%%%

%%% Use this command to include your bibliography file.

\bibliography{mybibfile}

\begin{thebibliography}{59}
\providecommand{\natexlab}[1]{#1}
\providecommand{\url}[1]{\texttt{#1}}
\expandafter\ifx\csname urlstyle\endcsname\relax
  \providecommand{\doi}[1]{doi: #1}\else
  \providecommand{\doi}{doi: \begingroup \urlstyle{rm}\Url}\fi

\bibitem[Askr et~al.()Askr, Elgeldawi, Aboul~Ella, Elshaier, Gomaa, and Hassanien]{askr2023deep}
H.~Askr, E.~Elgeldawi, H.~Aboul~Ella, Y.~A. Elshaier, M.~M. Gomaa, and A.~E. Hassanien.
\newblock Deep learning in drug discovery: an integrative review and future challenges.
\newblock \emph{Artificial Intelligence Review}, pages 5975--6037.

\bibitem[Austin et~al.(2021)Austin, Johnson, Ho, Tarlow, and Van Den~Berg]{austin2021structured}
J.~Austin, D.~D. Johnson, J.~Ho, D.~Tarlow, and R.~Van Den~Berg.
\newblock Structured denoising diffusion models in discrete state-spaces.
\newblock \emph{Proc. of NeurIPS}, pages 17981--17993, 2021.

\bibitem[Bilodeau et~al.(2022)Bilodeau, Jin, Jaakkola, Barzilay, and Jensen]{bilodeau2022generative}
C.~Bilodeau, W.~Jin, T.~Jaakkola, R.~Barzilay, and K.~F. Jensen.
\newblock Generative models for molecular discovery: Recent advances and challenges.
\newblock \emph{Wiley Interdisciplinary Reviews: Computational Molecular Science}, page e1608, 2022.

\bibitem[Cao et~al.(2025)Cao, Liu, Lu, Yao, and Li]{cao2025instructmol}
H.~Cao, Z.~Liu, X.~Lu, Y.~Yao, and Y.~Li.
\newblock Instructmol: Multi-modal integration for building a versatile and reliable molecular assistant in drug discovery.
\newblock In \emph{COLING}, 2025.

\bibitem[Christofidellis et~al.(2023)Christofidellis, Giannone, Born, Winther, Laino, and Manica]{christofidellis2023unifying}
D.~Christofidellis, G.~Giannone, J.~Born, O.~Winther, T.~Laino, and M.~Manica.
\newblock Unifying molecular and textual representations via multi-task language modelling.
\newblock In \emph{Proc. of ICML}, pages 6140--6157, 2023.

\bibitem[Dai et~al.(2018)Dai, Tian, Dai, Skiena, and Song]{dai2018syntax}
H.~Dai, Y.~Tian, B.~Dai, S.~Skiena, and L.~Song.
\newblock Syntax-directed variational autoencoder for structured data.
\newblock In \emph{Proc. of ICLR}, 2018.

\bibitem[De~Cao and Kipf(2018)]{de2018molgan}
N.~De~Cao and T.~Kipf.
\newblock {MolGAN: An implicit generative model for small molecular graphs}.
\newblock \emph{ICML 2018 workshop on Theoretical Foundations and Applications of Deep Generative Models}, 2018.

\bibitem[Dickson and Gagnon(2009)]{dickson2009cost}
M.~Dickson and J.~P. Gagnon.
\newblock The cost of new drug discovery and development.
\newblock \emph{Discovery medicine}, pages 172--179, 2009.

\bibitem[Druchok et~al.(2021)Druchok, Yarish, Gurbych, and Maksymenko]{druchok2021toward}
M.~Druchok, D.~Yarish, O.~Gurbych, and M.~Maksymenko.
\newblock Toward efficient generation, correction, and properties control of unique drug-like structures.
\newblock \emph{Journal of Computational Chemistry}, 2021.

\bibitem[Durant et~al.(2002)Durant, Leland, Henry, and Nourse]{durant2002reoptimization}
J.~L. Durant, B.~A. Leland, D.~R. Henry, and J.~G. Nourse.
\newblock Reoptimization of mdl keys for use in drug discovery.
\newblock \emph{Journal of chemical information and computer sciences}, pages 1273--1280, 2002.

\bibitem[Edwards et~al.(2021)Edwards, Zhai, and Ji]{edwards2021text2mol}
C.~Edwards, C.~Zhai, and H.~Ji.
\newblock Text2mol: Cross-modal molecule retrieval with natural language queries.
\newblock In \emph{Proc. of EMNLP}, 2021.

\bibitem[Edwards et~al.(2022)Edwards, Lai, Ros, Honke, Cho, and Ji]{edwards2022translation}
C.~Edwards, T.~Lai, K.~Ros, G.~Honke, K.~Cho, and H.~Ji.
\newblock Translation between molecules and natural language.
\newblock In \emph{Proc. of EMNLP}, 2022.

\bibitem[Fang et~al.()Fang, Liang, Zhang, Liu, Huang, Chen, Fan, and Chen]{fang2023mol}
Y.~Fang, X.~Liang, N.~Zhang, K.~Liu, R.~Huang, Z.~Chen, X.~Fan, and H.~Chen.
\newblock Mol-instructions: A large-scale biomolecular instruction dataset for large language models.
\newblock In \emph{Proc. of ICLR}.

\bibitem[G{\'o}mez-Bombarelli et~al.(2018)G{\'o}mez-Bombarelli, Wei, Duvenaud, Hern{\'a}ndez-Lobato, S{\'a}nchez-Lengeling, Sheberla, Aguilera-Iparraguirre, Hirzel, Adams, and Aspuru-Guzik]{gomez2018automatic}
R.~G{\'o}mez-Bombarelli, J.~N. Wei, D.~Duvenaud, J.~M. Hern{\'a}ndez-Lobato, B.~S{\'a}nchez-Lengeling, D.~Sheberla, J.~Aguilera-Iparraguirre, T.~D. Hirzel, R.~P. Adams, and A.~Aspuru-Guzik.
\newblock Automatic chemical design using a data-driven continuous representation of molecules.
\newblock \emph{ACS central science}, pages 268--276, 2018.

\bibitem[Gong et~al.(2024)Gong, Liu, Wu, and Wang]{gong2024text}
H.~Gong, Q.~Liu, S.~Wu, and L.~Wang.
\newblock Text-guided molecule generation with diffusion language model.
\newblock \emph{Proc. of AAAI}, pages 109--117, 2024.

\bibitem[Grisoni et~al.(2020)Grisoni, Moret, Lingwood, and Schneider]{grisoni2020bidirectional}
F.~Grisoni, M.~Moret, R.~Lingwood, and G.~Schneider.
\newblock Bidirectional molecule generation with recurrent neural networks.
\newblock \emph{Journal of chemical information and modeling}, pages 1175--1183, 2020.

\bibitem[Hartenfeller and Schneider(2011)]{hartenfeller2011novo}
M.~Hartenfeller and G.~Schneider.
\newblock De novo drug design.
\newblock \emph{Chemoinformatics and computational chemical biology}, pages 299--323, 2011.

\bibitem[He et~al.(2023)He, Sun, Tang, Wang, Huang, and Qiu]{he2023diffusionbert}
Z.~He, T.~Sun, Q.~Tang, K.~Wang, X.-J. Huang, and X.~Qiu.
\newblock Diffusionbert: Improving generative masked language models with diffusion models.
\newblock In \emph{Proc. of ACL}, pages 4521--4534, 2023.

\bibitem[Ho et~al.(2020)Ho, Jain, and Abbeel]{ho2020denoising}
J.~Ho, A.~Jain, and P.~Abbeel.
\newblock Denoising diffusion probabilistic models.
\newblock \emph{Proc. of NeurIPS}, pages 6840--6851, 2020.

\bibitem[Hou et~al.(2022)Hou, Liu, Cen, Dong, Yang, Wang, and Tang]{hou2022graphmae}
Z.~Hou, X.~Liu, Y.~Cen, Y.~Dong, H.~Yang, C.~Wang, and J.~Tang.
\newblock Graphmae: Self-supervised masked graph autoencoders.
\newblock In \emph{KDD}, 2022.

\bibitem[Huang et~al.(2023)Huang, Sun, Du, and Lv]{huang2023conditional}
H.~Huang, L.~Sun, B.~Du, and W.~Lv.
\newblock Conditional diffusion based on discrete graph structures for molecular graph generation.
\newblock In \emph{Proc. of AAAI}, pages 4302--4311, 2023.

\bibitem[Irwin et~al.(2020)Irwin, Tang, Young, Dandarchuluun, Wong, Khurelbaatar, Moroz, Mayfield, and Sayle]{irwin2020zinc20}
J.~J. Irwin, K.~G. Tang, J.~Young, C.~Dandarchuluun, B.~R. Wong, M.~Khurelbaatar, Y.~S. Moroz, J.~Mayfield, and R.~A. Sayle.
\newblock Zinc20—a free ultralarge-scale chemical database for ligand discovery.
\newblock \emph{Journal of chemical information and modeling}, pages 6065--6073, 2020.

\bibitem[Jiang et~al.(2022)Jiang, Zhang, Zhao, Ma, Liu, Yuan, and Niu]{jiang2022multigran}
J.~Jiang, R.~Zhang, Z.~Zhao, J.~Ma, Y.~Liu, Y.~Yuan, and B.~Niu.
\newblock Multigran-smiles: multi-granularity smiles learning for molecular property prediction.
\newblock \emph{Bioinformatics}, pages 4573--4580, 2022.

\bibitem[Jin et~al.(2018)Jin, Barzilay, and Jaakkola]{jin2018junction}
W.~Jin, R.~Barzilay, and T.~Jaakkola.
\newblock Junction tree variational autoencoder for molecular graph generation.
\newblock In \emph{Proc. of ICML}, 2018.

\bibitem[Jo et~al.(2022)Jo, Lee, and Hwang]{jo2022score}
J.~Jo, S.~Lee, and S.~J. Hwang.
\newblock Score-based generative modeling of graphs via the system of stochastic differential equations.
\newblock In \emph{Proc. of ICML}, pages 10362--10383, 2022.

\bibitem[Kong et~al.(2023)Kong, Cui, Sun, Zhuang, Prakash, and Zhang]{kong2023autoregressive}
L.~Kong, J.~Cui, H.~Sun, Y.~Zhuang, B.~A. Prakash, and C.~Zhang.
\newblock Autoregressive diffusion model for graph generation.
\newblock In \emph{Proc. of ICML}, pages 17391--17408, 2023.

\bibitem[Krenn et~al.(2020)Krenn, H{\"a}se, Nigam, Friederich, and Aspuru-Guzik]{krenn2020self}
M.~Krenn, F.~H{\"a}se, A.~Nigam, P.~Friederich, and A.~Aspuru-Guzik.
\newblock Self-referencing embedded strings (selfies): A 100\% robust molecular string representation.
\newblock \emph{Machine Learning: Science and Technology}, 2020.

\bibitem[Li et~al.(2024)Li, Liu, Fan, Wei, Liu, Tang, and Li]{li2024empowering}
J.~Li, Y.~Liu, W.~Fan, X.-Y. Wei, H.~Liu, J.~Tang, and Q.~Li.
\newblock Empowering molecule discovery for molecule-caption translation with large language models: A chatgpt perspective.
\newblock \emph{IEEE transactions on knowledge and data engineering}, 36\penalty0 (11):\penalty0 6071--6083, 2024.

\bibitem[Li et~al.(2025)Li, Liu, Ding, Fan, Li, and Li]{li2025large}
J.~Li, W.~Liu, Z.~Ding, W.~Fan, Y.~Li, and Q.~Li.
\newblock Large language models are in-context molecule learners.
\newblock \emph{IEEE Transactions on Knowledge and Data Engineering}, 2025.

\bibitem[Li et~al.(2018)Li, Zhang, and Liu]{li2018multi}
Y.~Li, L.~Zhang, and Z.~Liu.
\newblock Multi-objective de novo drug design with conditional graph generative model.
\newblock \emph{Journal of cheminformatics}, 2018.

\bibitem[Liu et~al.(2024{\natexlab{a}})Liu, Xu, Luo, and Jiang]{liu2024graph}
G.~Liu, J.~Xu, T.~Luo, and M.~Jiang.
\newblock Graph diffusion transformers for multi-conditional molecular generation.
\newblock \emph{Proc. of NeurIPS}, 2024{\natexlab{a}}.

\bibitem[Liu et~al.(2024{\natexlab{b}})Liu, Ren, Tao, and Ren]{liu2024git}
P.~Liu, Y.~Ren, J.~Tao, and Z.~Ren.
\newblock Git-mol: A multi-modal large language model for molecular science with graph, image, and text.
\newblock \emph{Computers in Biology and Medicine}, page 108073, 2024{\natexlab{b}}.

\bibitem[Liu et~al.(2023{\natexlab{a}})Liu, Nie, Wang, Lu, Qiao, Liu, Tang, Xiao, and Anandkumar]{liu2023multi}
S.~Liu, W.~Nie, C.~Wang, J.~Lu, Z.~Qiao, L.~Liu, J.~Tang, C.~Xiao, and A.~Anandkumar.
\newblock Multi-modal molecule structure--text model for text-based retrieval and editing.
\newblock \emph{Nature Machine Intelligence}, 2023{\natexlab{a}}.

\bibitem[Liu et~al.(2019)Liu, Ott, Goyal, Du, Joshi, Chen, Levy, Lewis, Zettlemoyer, and Stoyanov]{liu2019roberta}
Y.~Liu, M.~Ott, N.~Goyal, J.~Du, M.~Joshi, D.~Chen, O.~Levy, M.~Lewis, L.~Zettlemoyer, and V.~Stoyanov.
\newblock Roberta: A robustly optimized bert pretraining approach.
\newblock \emph{arXiv preprint arXiv:1907.11692}, 2019.

\bibitem[Liu et~al.(2023{\natexlab{b}})Liu, Li, Luo, Fei, Cao, Kawaguchi, Wang, and Chua]{liu2023molca}
Z.~Liu, S.~Li, Y.~Luo, H.~Fei, Y.~Cao, K.~Kawaguchi, X.~Wang, and T.-S. Chua.
\newblock Molca: Molecular graph-language modeling with cross-modal projector and uni-modal adapter.
\newblock In \emph{Proc. of EMNLP}, 2023{\natexlab{b}}.

\bibitem[Ma and Zhang(2021)]{ma2021gf}
C.~Ma and X.~Zhang.
\newblock Gf-vae: a flow-based variational autoencoder for molecule generation.
\newblock In \emph{Proc. of CIKM}, pages 1181--1190, 2021.

\bibitem[Murphy et~al.(2019)Murphy, Srinivasan, Rao, and Ribeiro]{murphy2019relational}
R.~Murphy, B.~Srinivasan, V.~Rao, and B.~Ribeiro.
\newblock Relational pooling for graph representations.
\newblock In \emph{Proc. of ICML}, pages 4663--4673, 2019.

\bibitem[Nag et~al.(2022)Nag, Baidya, Mandal, Mathew, Das, Devi, and Kumar]{nag2022deep}
S.~Nag, A.~T. Baidya, A.~Mandal, A.~T. Mathew, B.~Das, B.~Devi, and R.~Kumar.
\newblock Deep learning tools for advancing drug discovery and development.
\newblock \emph{3 Biotech}, page 110, 2022.

\bibitem[Nikolentzos and Vazirgiannis(2020)]{nikolentzos2020random}
G.~Nikolentzos and M.~Vazirgiannis.
\newblock Random walk graph neural networks.
\newblock \emph{Proc. of NeurIPS}, pages 16211--16222, 2020.

\bibitem[Niu et~al.(2020)Niu, Song, Song, Zhao, Grover, and Ermon]{niu2020permutation}
C.~Niu, Y.~Song, J.~Song, S.~Zhao, A.~Grover, and S.~Ermon.
\newblock Permutation invariant graph generation via score-based generative modeling.
\newblock In \emph{Proc. of AISTATS}, pages 4474--4484, 2020.

\bibitem[Pei et~al.()Pei, Zhang, Zhu, Wu, Gao, Wu, Xia, and Yan]{pei2023biot5}
Q.~Pei, W.~Zhang, J.~Zhu, K.~Wu, K.~Gao, L.~Wu, Y.~Xia, and R.~Yan.
\newblock Biot5: Enriching cross-modal integration in biology with chemical knowledge and natural language associations.
\newblock In \emph{Proc. of EMNLP}.

\bibitem[Rogers and Hahn(2010)]{rogers2010extended}
D.~Rogers and M.~Hahn.
\newblock Extended-connectivity fingerprints.
\newblock \emph{Journal of chemical information and modeling}, pages 742--754, 2010.

\bibitem[Schneider and Fechner(2005)]{schneider2005computer}
G.~Schneider and U.~Fechner.
\newblock Computer-based de novo design of drug-like molecules.
\newblock \emph{Nature Reviews Drug Discovery}, pages 649--663, 2005.

\bibitem[Schneider et~al.(2015)Schneider, Sayle, and Landrum]{schneider2015get}
N.~Schneider, R.~A. Sayle, and G.~A. Landrum.
\newblock Get your atoms in order - an open-source implementation of a novel and robust molecular canonicalization algorithm.
\newblock \emph{Journal of chemical information and modeling}, pages 2111--2120, 2015.

\bibitem[Simonovsky and Komodakis(2018)]{simonovsky2018graphvae}
M.~Simonovsky and N.~Komodakis.
\newblock Graphvae: Towards generation of small graphs using variational autoencoders.
\newblock In \emph{Proc. of ICANN}, 2018.

\bibitem[Su et~al.(2022)Su, Du, Yang, Zhou, Li, Rao, Sun, Lu, and Wen]{su2022molecular}
B.~Su, D.~Du, Z.~Yang, Y.~Zhou, J.~Li, A.~Rao, H.~Sun, Z.~Lu, and J.-R. Wen.
\newblock A molecular multimodal foundation model associating molecule graphs with natural language.
\newblock \emph{arXiv preprint arXiv:2209.05481}, 2022.

\bibitem[Vignac et~al.(2023)Vignac, Krawczuk, Siraudin, Wang, Cevher, and Frossard]{vignac2023digress}
C.~Vignac, I.~Krawczuk, A.~Siraudin, B.~Wang, V.~Cevher, and P.~Frossard.
\newblock Digress: Discrete denoising diffusion for graph generation.
\newblock In \emph{ICLR}, 2023.

\bibitem[Weininger(1988)]{weininger1988smiles}
D.~Weininger.
\newblock Smiles, a chemical language and information system. 1. introduction to methodology and encoding rules.
\newblock \emph{Journal of chemical information and computer sciences}, pages 31--36, 1988.

\bibitem[Weiss et~al.(2023)Weiss, Mayo~Yanes, Chakraborty, Cosmo, Bronstein, and Gershoni-Poranne]{weiss2023guided}
T.~Weiss, E.~Mayo~Yanes, S.~Chakraborty, L.~Cosmo, A.~M. Bronstein, and R.~Gershoni-Poranne.
\newblock Guided diffusion for inverse molecular design.
\newblock \emph{Nature Computational Science}, pages 873--882, 2023.

\bibitem[White(2020)]{white2020pubmed}
J.~White.
\newblock Pubmed 2.0.
\newblock \emph{Medical reference services quarterly}, 2020.

\bibitem[Wu et~al.(2023)Wu, Tang, Sun, and Xiong]{wu2023molecular}
T.~Wu, Y.~Tang, Q.~Sun, and L.~Xiong.
\newblock Molecular joint representation learning via multi-modal information of smiles and graphs.
\newblock \emph{IEEE/ACM Transactions on Computational Biology and Bioinformatics}, 2023.

\bibitem[Xiang et~al.(2024)Xiang, Zhao, Ma, and Deng]{xiang2024instruction}
Y.~Xiang, H.~Zhao, C.~Ma, and Z.-H. Deng.
\newblock Instruction-based molecular graph generation with unified text-graph diffusion model.
\newblock \emph{arXiv preprint arXiv:2408.09896}, 2024.
\newblock Full version of this paper.

\bibitem[Ye et~al.(2025)Ye, Cai, Lai, Wang, Huang, Wang, Liu, and Zeng]{ye2025drugassist}
G.~Ye, X.~Cai, H.~Lai, X.~Wang, J.~Huang, L.~Wang, W.~Liu, and X.~Zeng.
\newblock Drugassist: A large language model for molecule optimization.
\newblock \emph{Briefings in Bioinformatics}, 26\penalty0 (1):\penalty0 bbae693, 2025.

\bibitem[Ying et~al.(2021)Ying, Cai, Luo, Zheng, Ke, He, Shen, and Liu]{ying2021transformers}
C.~Ying, T.~Cai, S.~Luo, S.~Zheng, G.~Ke, D.~He, Y.~Shen, and T.-Y. Liu.
\newblock Do transformers really perform badly for graph representation?
\newblock \emph{Proc. of NeurIPS}, pages 28877--28888, 2021.

\bibitem[You et~al.(2018)You, Ying, Ren, Hamilton, and Leskovec]{you2018graphrnn}
J.~You, R.~Ying, X.~Ren, W.~Hamilton, and J.~Leskovec.
\newblock Graphrnn: Generating realistic graphs with deep auto-regressive models.
\newblock In \emph{Proc. of ICML}, pages 5708--5717, 2018.

\bibitem[Zang and Wang(2020)]{zang2020moflow}
C.~Zang and F.~Wang.
\newblock Moflow: an invertible flow model for generating molecular graphs.
\newblock In \emph{Proc. of KDD}, pages 617--626, 2020.

\bibitem[Zhao et~al.()Zhao, Ma, Zhang, Deng, and Wei]{zhao2023more}
H.~Zhao, S.~Ma, D.~Zhang, Z.-H. Deng, and F.~Wei.
\newblock Are more layers beneficial to graph transformers?
\newblock In \emph{Proc. of ICLR}.

\bibitem[Zhao et~al.(2024)Zhao, Liu, Chang, Xu, Fu, Deng, Kong, and Liu]{zhao2024gimlet}
H.~Zhao, S.~Liu, M.~Chang, H.~Xu, J.~Fu, Z.~Deng, L.~Kong, and Q.~Liu.
\newblock Gimlet: A unified graph-text model for instruction-based molecule zero-shot learning.
\newblock \emph{Proc. of NeurIPS}, 2024.

\bibitem[Zhu et~al.(2024)Zhu, Xiao, and Honavar]{zhu20243m}
H.~Zhu, T.~Xiao, and V.~G. Honavar.
\newblock 3m-diffusion: Latent multi-modal diffusion for language-guided molecular structure generation.
\newblock In \emph{First Conference on Language Modeling}, 2024.

\end{thebibliography}

%%%%%%%%%%%%%%%%%%%%%%%%%%%%%%%%%%%%%%%%%%%%%%%%%%%%%%%%%%%%%%%%%%%%%%%%

\appendix

\newpage

\section{Analysis of FCD}
\label{sec:appendix1}

FCD is a metric originally used to measure the distance between two sets of distributions. Previous works employed this metric to compare the distance between the ground-truth test sets and generation sets in instruction-based molecule generation tasks, rather than comparing the similarity one by one as other metrics do.

Since this metric compares the distance between molecule distributions, it can be misled by molecules generated from other instructions. For example, FCD is equivariant under reordering, meaning it cannot distinguish errors where generated molecules and instructions are exactly cross-matched. For instance, the instruction $S_1$ match the molecule ‘CCO’, and instruction $S_2$ match the molecule ‘[He]’.  If we generate ‘CC’ with instruction $S_1$ and ‘[He]’ with instruction $S_2$, there's no doubt that FCD ([‘CCO’, ‘[He]’], [‘CC’, ‘[He]’]) > 0. However, If we generate ‘[He]’ with instruction $S_1$ and ‘CCO’ with instruction $S_2$, which is a totally irrelavant results, we'll got an unexpected result FCD ([‘CCO’, ‘[He]’], [‘[He]’, ‘CCO’]) = 0.

To demonstrate the effectiveness of our model in FCD, we divided the test set into 10 subsets and calculated the average and variance FCD across them. This method can partially reduce interference between different instructions. We include tgm-dlm for comparison, as it has a lower FCD, while MolXPT is not open-sourced.

Using this processing method, our model achieved better FCD results, as shown in Table \ref{table20}. The difference in magnitude between the two settings is an inherent property of FCD, which is related to the number of molecule sets.

\section{Proof of equivarient property}
\label{sec:appendix2}

The graph generation model needs to ensure that the node's attribution in the graph is equivariant to reorderings, meaning that the matrices under reordering can represent the same graph. This property requires no positional embedding in the denoising network. However, we empirically found the importance of positional embedding for better performance.

We attribute this result to the rich position information in the text description, as shown in Figure \ref{fig:example-pos}. So, the model will lose its perception of this essential information without position embedding. This issue is more significant in diffusion models since the located step must be later than the decoding of the atom.

\begin{figure}[htbp]
  \centering
  \includegraphics[scale=0.5]{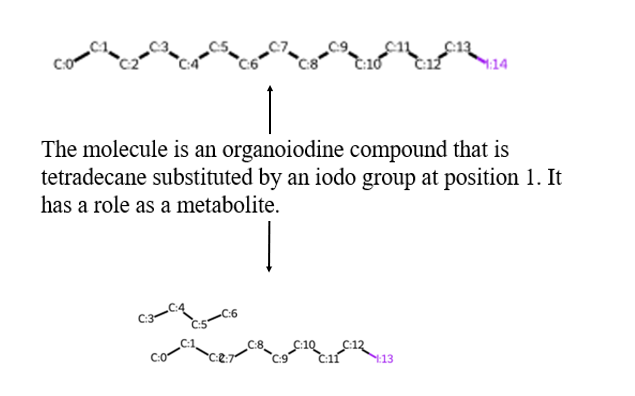}
  \caption{An example of generation results w/w.o position embedding. Adding position embedding can locate the position of the iodo group. }
  \label{fig:example-pos}
\end{figure}

Therefore, to ensure equivariance on the inference step while retaining the utilization of position embedding, We adopt the following strategy: We use position embedding as before during training and make a random permutation of the position index in sequence during inference. It can be proven that the model is still equivariant during generation, although it is trained to fit a specific order of decoding atoms.

The theorem to be proven in this section can be written as follows:

\begin{theorem}{ (Equivariancy for Graph Generation)}
\label{theorem3.1}
For any permutation $\pi$, the model will generate a graph with node feature $V$ and adjacency matrix $A$ with input position index $i$, satisfying $P_i (V, A) = P _{\pi^T i} (\pi^T V, \pi^TA\pi)$.
\end{theorem}

The proof aims to show that: If the model's network architecture is equivariant with input mapping permutation and the training loss is permutation invariant, the model will generate a distribution reordering with position index.

\begin{lemma}
The loss in training is invariant under reordering: For given predict graph $\hat G$ and ground truth graph $G$, given a permutation$\pi$, we have$l (\pi. \hat G, \pi. G) = l (\hat G, G)$
\end{lemma}

\begin{proof}
Since the defined loss function is the same for each node and edge, we have 
\begin{align*}
l(\pi.{\hat{G}},\pi.G) = \sum_{i}l_{X}(\pi.{\hat{X}}_{i},x_{\pi^{-1}(i)}) \\ + \sum_{i,j}l_{E}(\pi.{\hat{\mathbb{E}}}_{i,j},e_{\pi^{-1}(i),\pi^{-1}(j)}) \\ = \sum_{i}l_{X}({\hat{X}}_{i},x_{i})+\sum_{i,j}l_{E}({\hat{\mathbb{E}}}_{i,j},e_{i,j}) \\ = l(\hat G, G)
\end{align*}
\end{proof}

\begin{lemma}
The architecture of the model is equivariant with the input embedding: For given predict graph $\hat G$ and ground truth graph $G$, given a permutation$\pi$, we have $\phi_{\theta}(\pi.G^{t})=\pi.\phi_{\theta}(G^{t})$
\end{lemma}

\begin{proof}
Define $G_t = (X_t, E_t)$ as the noisy graph, $(\pi.X_t, \pi.E_t)$ as the permutation. Since the input is permutation equivariant, and we have:

\begin{itemize}
    \item The self attention architecture is permutation invariant
    \item The linear layers are permutation invariant.
    \item The Layer normalization is permutation equivariant.
\end{itemize}

Therefore, the model is a combination of permutation equivariant blocks.
\end{proof}

\begin{proof}
proof of theorem \ref{theorem3.1}

For the input embedding after model permutation, we have $\pi^T X + \pi^T emb(i) = \pi^T (X+emb(i)) $, Therefore, the assumption of input mapping permutation equivariant holds.

Following the proof in DiGress, with the lemma: if a distribution $p(G_T)$ is invariant to the action of a group G and the transition probabilities $p(G_{t-1}|G_t)$ are equivariant, them $p(G_0)$is invariant to the action of G. We apply this result to the special case of permutations:

\begin{itemize}
    \item The initial noise distribution is the mask distribution on each node and edge. It is therefore permutation invariant.
    \item The denoising neural networks is permutation equivariant.
    \item The transition probabilities function $p_\theta(G_{t-1}|G_t) = \sum_G q(G_{t-1}, G|G_t)\hat p_\theta(G)$ is equivariant to $\bar p_\theta(G)$ and $G_T$.
\end{itemize}

The conditions are therefore satisfied, and the model satisfies $P_i(V, A) = P _{\pi^T i}(\pi^T V, \pi^TA\pi)$.

\end{proof}

\begin{table}
\fontsize{8}{12}\selectfont
\begin{center}
\begin{tabular}{c|c c}
\toprule
 Model & 10 subset (new) & 1 set (origin) \\
\midrule

Tgm-dlm	& $3.38 \pm 0.078$ & 0.77 \\
UTGDiff	& $3.25 \pm 0.055$ &	0.866 \\

\bottomrule

\end{tabular}
\caption{FCD comparison}
\label{table20}
\end{center}
\end{table}

\section{Details for evaluation metrics}
\label{sec:appendix3}

Since the model generates a matrix of nodes and edges, although ultimately a SMILES will be parsed based on the matrix, it is not directly generated, which can easily cause the problem of different strings for same graph. Therefore, although molecules can be represented by biological sequence structures, and previous models have also established evaluation indicators from the level of string similarity, this article does not involve these indicators in comparison, including NLP indicators such as BLEU, Levenshtein.

Therefore, we use a series of metrics related to the similarity of molecular graphs:

\begin{itemize}
    \item Exact: Whether the two molecules are same.
    \item Valid: Whether the generated molecule satisfied the constraint for molecule, such as the valence rule.
    \item FTS: We employ three fingerprint metrics: MACCS FTS, RDK FTS, and Morgan FTS, where FTS stands for fingerprint Tanimoto similarity. MACCS \citep{durant2002reoptimization}, RDK\citep{schneider2015get} and Morgan\citep{rogers2010extended}. The fingerprints of two molecules are compared using Tanimoto similarity (also known as Jaccard index), and the average similarity over the evaluation dataset is reported. We use RDKIT toolkit.
    \item FCD score (Fréchet chemnet distance): Measure molecular similarity based on a pre-trained "ChemNet" bioinformatics network. We use fcd 1.1 in python.
\end{itemize}

For better evaluation, we also list some metrics which is not discussed in previous baseline. Here're the information and results:

\begin{itemize}
    \item Maximum Common Edge Substructure (MCES) : This metric finds the largest subgraph (in terms of the number of edges) that is common to two or more given graphs. We used the rdRascalMCES function in RDKit (Python) to measure MCES, with the default threshold set to 0.7. If either of the Johnson similarity estimates falls below this threshold, we consider as invalid for MCES.
\end{itemize}

The results of MCES in CHEBI-20 is shown in Table \ref{table30}.

\begin{table}
\fontsize{8}{12}\selectfont
\begin{center}
\begin{tabular}{c|c }
\toprule
Method	&MCES\\
\midrule
BioT5 (reproduce)	&0.8891\\
gpt4-2024-04-09 (10 shot+BM25)	&0.9060\\
TGM-DLM &0.9036\\
UTGDiff (our) &\textbf{0.9166}\\

\bottomrule

\end{tabular}
\caption{MCES results in different models}
\label{table30}
\end{center}
\end{table}

\section{Details for experiments}
\label{sec:appendix4}

Here're the hyperparam for three different tasks describe in article, and here's some general information on hyperparam: 

(1) The pretraining process spans nearly 300K steps and is executed on four NVIDIA 24GB  GeForce RTX 3090 GPUs with batch size 32 per GPU, totally trained for near 1 week. The finetuning process spans nearly 800K steps and is executed on two NVIDIA 24GB  GeForce RTX 3090 GPUs with batch size 16 per GPU, totally trained for near 3 days.

(2) During finetuning, we gradually increase the accumulation step for the tradeoff of training efficiency and convergence level. The effective of this method is shown in the ablation of accumulation step shown in Table \ref{table8}. So, the initial accumulation step is 1, and it will be increase to 64 finally.

(3) When the model is used for generation, we additionally introduces top-k sampling to help improve the quality of generation: when a class is taken from the calculated probability for discrete generation, only the larger node class is taken. 

(4) All the training and generation program is running under specific seed. We only generate onces, but experiments shows there's no significant variance during sampling for different seeds. We show the 3 seed example in forward reaction prediction in Table \ref{table12}

The specific data is listed in Table \ref{table9}, \ref{table10}, \ref{table11}

\paragraph{prompt} Here we list the prompt for the three datasets: For CHEBI-20 dataset, We give its instruction form as : "'[molecule description]' is the description of molecule:" For the two editing dataset, since there's task instruction in datasets, we don't use any additional instructions. Here's an example in Retrosynthesis dataset: "Please suggest potential reactants for the given product."; Here's an example in forward reaction prediction dataset: "With the provided reactants and reagents, propose a potential product." 

Also, here we list the prompt for large language model baselines. To limit the Generalist LLM to generate the desired SMILES, we use below prompts for both chatgpt and gpt-4 baseline:

'role': user, 'content':  'What's the SMILES of the molecular with these properties? \textbackslash n' + (task instructions)+ 'The format of your answer must be "SMILES:", with no extra words.'

\begin{table}
\fontsize{8}{12}\selectfont
\begin{center}
\begin{tabular}{c|c c c c c c}
\toprule
Model & MACCS $\uparrow$ & RDK $\uparrow$ & Morgan $\uparrow$ & FCD $\downarrow$ & Exact $\uparrow$ & Valid $\uparrow$ \\
\midrule
baseline & 0.867 & 0.763 & 0.695 & 0.923 & 0.227 & 0.856 \\

acc\_step = 4 & 0.838 & 0.719 & 0.629 & 1.651 & 0.149 & 0.792 \\
\bottomrule

\end{tabular}
\caption{ablation of accumulation step (from-scratch)}
\label{table8}
\end{center}
\end{table}

\begin{table}
\fontsize{8}{12}\selectfont
\begin{center}
\begin{tabular}{c |c }
\toprule
param name & value \\
\midrule

learning rate & 5e-5 \\ 
batch size & 16 \\
accumulation step & (1,4,16,64) \\
accumulation update epoch & (1,4,16,64) \\
top k & 15 \\
predict molecule length & 128 \\
seed & 42 \\
step size & 20 \\

\bottomrule

\end{tabular}
\caption{Hyperparam for CHEBI-20 datasets} 
\label{table9} 
\end{center}
\end{table}

\begin{table}
\fontsize{8}{12}\selectfont
\begin{center}

\begin{tabular}{c |c }
\toprule
param name & value \\
\midrule

learning rate & 5e-5 \\ 
batch size & 16 \\
accumulation step & (1,4,16,64) \\
epoch & 1000 \\
accumulation update epoch & (90,150,180) \\
top k & 15 \\
predict molecule length & 108 \\
seed & 42 \\
step size & 10 \\

\bottomrule

\end{tabular}
\caption{Hyperparam for Retrosynthesis datasets} 
\label{table10} 
\end{center}
\end{table}

\begin{table}
\fontsize{8}{12}\selectfont
\begin{center}
\begin{tabular}{c |c }
\toprule
param name & value \\
\midrule

learning rate & 5e-5 \\ 
batch size & 16 \\
accumulation step & (1,4,16,64) \\
epoch & 210 \\
accumulation update epoch & (90,150,180) \\
top k & 15 \\
predict molecule length & 96 \\
seed & 42 \\
step size & 10 \\

\bottomrule

\end{tabular}
\caption{Hyperparam for forward reaction prediction datasets} 
\label{table11} 
\end{center}
\end{table}

\begin{table}
\fontsize{8}{12}\selectfont
\begin{center}
\begin{tabular}{c c|c c c c c}
\toprule
Type & Model & MACCS & RDK & Morgan & Exact & Valid \\
\midrule

& seed 0 & 0.971 & 0.938 & 0.938 & 0.821 & 0.973 \\

& seed 1 & 0.974 & 0.940 & 0.938 & 0.825 & 0.968 \\

& seed 42 & 0.973 & 0.943 & 0.942 & 0.825 & 0.972 \\
\bottomrule

\end{tabular}
\caption{3 seeds for forward reaction prediction results}
\label{table12}
\end{center}
\end{table}

\begin{figure}[htbp]
  \centering
  \includegraphics[scale=0.2]{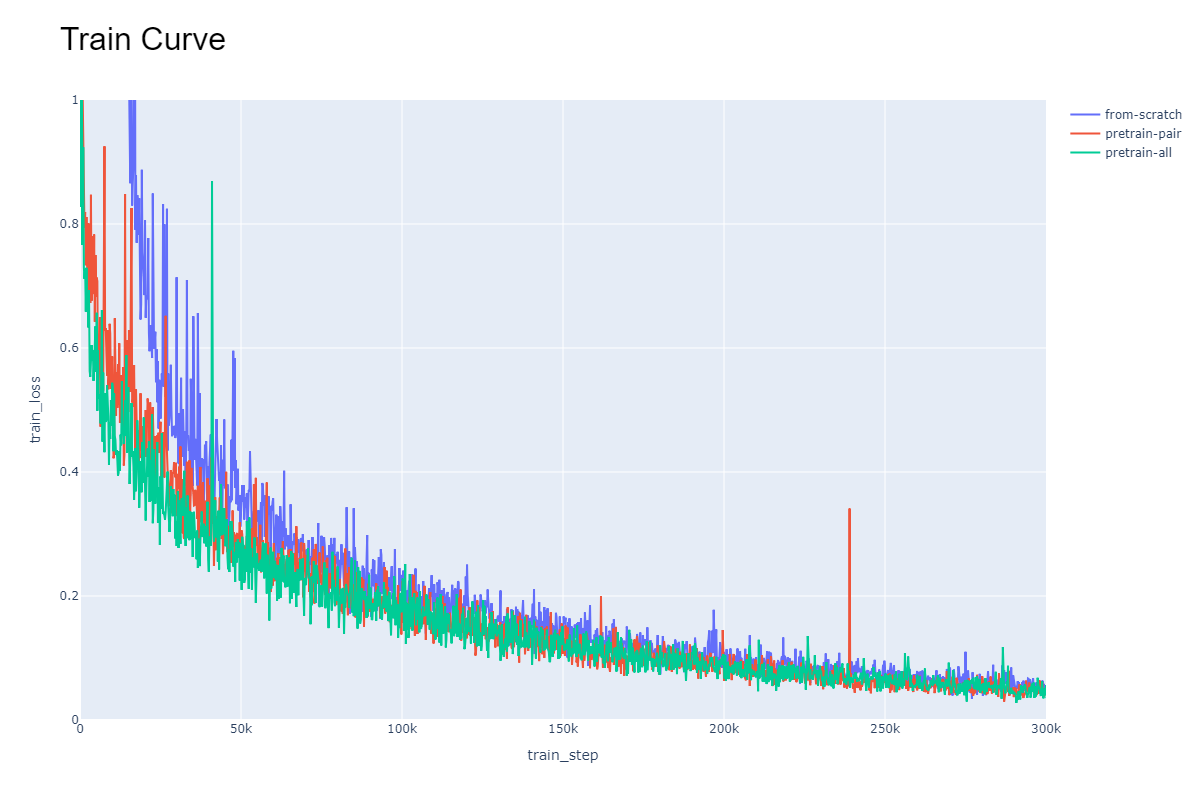}
  \caption{The training curve in first 300K steps. Pretraining also demonstrates lower initial loss and faster convergence compared to training models from scratch.}
  \label{fig:example-curve}
\end{figure}

\section{Algorithm Pseudo Code}
\label{sec:appendix5}

The pseudo code of the training and sampling algorithm is shown in \ref{alg:train} and \ref{alg:sample}.

\begin{algorithm}[!ht]
    \caption{Training Algorithm}
    \label{alg:train}
    \renewcommand{\algorithmicrequire}{\textbf{Input:}}
    \renewcommand{\algorithmicensure}{\textbf{Output:}}
    
    \begin{algorithmic}[1]
    
        \REQUIRE A graph $G = (V, E)$ with description $S$  %%input
        
        \REPEAT
        
            \STATE Sample $t \sim U(1, \cdot, T)$
    
            \STATE Sample $G^t \sim V\bar Q_{V}^{t} \times E\bar Q_{E}^{t}$
    
            \STATE $\hat{p}(S), \hat{p}(V), \hat{p}(E) = \phi_\theta(G^{t},S)$
    
            \STATE loss = $l_{CE}(\hat{p}(E),E) + l_{CE}(\hat{p}(V),V) + l_{CE}(\hat{p}(S),S)$
            
            \STATE optimizer.step(loss)

        \UNTIL converged
       
    \end{algorithmic}
\end{algorithm}

\begin{algorithm}[!ht]
    \caption{Sampling Algorithm}
    \label{alg:sample}
    \renewcommand{\algorithmicrequire}{\textbf{Input:}}
    \renewcommand{\algorithmicensure}{\textbf{Output:}}
    
    \begin{algorithmic}[1]

        \STATE $G^T$ with all masked nodes and edges
        
        \FOR{$t = T$ to $1$}
        
        \STATE $\hat{p}(V), \hat{p}(E) = \phi_\theta(G^{t},S)$

        \STATE $p_{\theta}(v^{t-1}_i\mid G^{t}, S)=\sum_{v}q(v^{t-1}_i \mid v_{i}=v,v^{t}_{i})\;{\hat{p}}_{i}(v)$

        \STATE $p_{\theta}(e^{t-1}_{ij}\mid G^{t}, S)=\sum_{e}q(e^{t-1}_{ij}\mid e_{ij}=e,e^{t}_{ij})\;{\hat{p}}_{ij}(e)$

        \STATE $G^{t-1} \sim \prod_i p_{\theta}(v^{t-1}_i\mid G^{t}, S) \prod_{ij} p_{\theta}(e^{t-1}_{ij}\mid G^{t}, S)$

        \ENDFOR

        \RETURN $G^0$
       
    \end{algorithmic}
\end{algorithm}

\section{Potential Risks}
\label{sec:appendix6}

Potential risks associated with molecular generation models include the inadvertent creation of toxic or unstable compounds, which could pose significant safety hazards. Additionally, biases in the training data could be learned by the model, leading to unintended consequences. Another risk is the possibility of the model being used to discover harmful molecules instead of beneficial ones. These concerns highlight the necessity for stringent regulatory frameworks and ethical guidelines to govern the use of these models.

\section{Source of codes and datasets}
\label{sec:appendix7}

Our Code is avaliable at: https://github.com/ran1812/UTGDiff

All the datasets we used are open-sorcued, can be founded in github or huggingface:

huggingface: ZINC (zpn/zinc20); Pubmed; Pubchem

github: Mol-instruction, CHEBI-20

\end{document}